\documentclass[journal,comsoc]{IEEEtran}
\IEEEoverridecommandlockouts
\usepackage[dvips]{graphicx}
\usepackage{epsfig}
\usepackage{hyperref}
\usepackage{breakurl}
\usepackage[font=footnotesize]{caption}
\usepackage{subcaption}
\DeclareGraphicsExtensions{.JPG,.eps,.pdf}
\usepackage{cite}
\usepackage{multirow,tabularx}
\usepackage{amsfonts}
\usepackage{amsmath}
\usepackage{amsthm}
\usepackage{amscd}
\usepackage{amssymb}
\usepackage{cleveref,url}
\usepackage{proof}
\usepackage{pgf,pgfarrows,pgfnodes,pgfautomata,pgfheaps}
\usepackage{cases}
\usepackage[neverdecrease]{paralist}
\usepackage{setspace}
\usepackage{footnote}
\usepackage{tabularx}
\usepackage{epstopdf}
\usepackage{algpseudocode}
\usepackage{algorithm}
\usepackage{placeins}
\usepackage{color}
\usepackage{bbm}
\usepackage{xcolor,colortbl}
\usepackage{hhline}
\usepackage{stfloats}
\usepackage{paralist}
\newlength\figureheight
\newlength\figurewidth
\newtheorem{thm}{Theorem}

\newtheorem{defn}{Definition}

\theoremstyle{remark}

\newcommand{\be}{\begin{equation}}
\newcommand{\ee}{\end{equation}}
\newcommand{\bea}{\begin{eqnarray}}
\newcommand{\eea}{\end{eqnarray}}
\newcommand{\ben}{\begin{enumerate}}
	\newcommand{\een}{\end{enumerate}}
\algdef{SE}[DOWHILE]{Do}{doWhile}{\algorithmicdo}[1]{\algorithmicwhile\ #1}%

\newcommand{\beseq}{\begin{subequations}}
	\newcommand{\eeseq}{\end{subequations}}

\ifCLASSINFOpdf
\else
\fi

\normalsize

\begin{document}
\title{Distributed Power Control for Large Energy Harvesting Networks: A Multi-Agent Deep Reinforcement Learning Approach \thanks{The work in this paper will appear in part at IEEE ICASSP 2019~\cite{MSharma_ICASSP_2019} and IEEE WiOpt 2019~\cite{MSharma_WiOpt_2019}. Mohit~K.~Sharma, Mohamad Assaad, and M\'erouane Debbah are with the CentraleSupelec, Universit\'e Paris-Saclay, 91192 Gif-sur-Y vette, France. (e-mails: \{mohitkumar.sharma, mohamad.assaad\}@centralesupelec.fr. A.\! Zappone was with CentraleSupelec, Gif-Sur-Yvette, France, and is now with the University of Cassino and Southern Lazio, Cassino, Italy (email: alessio.zappone@unicas.it). M\'erouane Debbah and Spyridon Vassilaras are  with the Mathematical and Algorithmic Sciences Lab, Huawei France R\&D, Paris, France (e-mails: \{merouane.debbah, spyros.vassilaras\}@huawei.com). This research has been partly supported by the ERC-PoC 727682 CacheMire project. The work of A. Zappone was supported by the H2020 MSCA IF BESMART, Grant 749336.}
	}

\author{\IEEEauthorblockN{Mohit~K.~Sharma, Alessio Zappone,~\IEEEmembership{Senior Member,~IEEE}, Mohamad~Assaad,~\IEEEmembership{Senior Member,~IEEE}, M\'erouane~ Debbah~\IEEEmembership{Fellow,~IEEE}, and Spyridon~Vassilaras~\IEEEmembership{Senior Member,~IEEE}}}

\maketitle
\begin{abstract}
In this paper, we develop a multi-agent reinforcement learning (MARL) framework to obtain online power control policies for a \emph{large} energy harvesting (EH) multiple access channel, when only causal information about the EH process and wireless channel is available. In the proposed framework, we model the online power control problem as a discrete-time mean-field game (MFG), and analytically show that the MFG has a \emph{unique} stationary solution. Next, we leverage the fictitious play property of the mean-field games, and the deep reinforcement learning technique to learn the stationary solution of the game, in a completely  \emph{distributed} fashion. We analytically show that the proposed procedure converges to the unique stationary solution of the MFG. This, in turn, ensures that the optimal policies can be learned in a completely distributed fashion. In order to benchmark the performance of the distributed policies, we also develop a deep neural network (DNN) based centralized as well as distributed online power control schemes. Our simulation results show the efficacy of the proposed power control policies. In particular, the DNN based centralized power control policies provide a very good performance for large EH networks for which the design of optimal policies is intractable using the conventional methods such as Markov decision processes. Further, performance of both the distributed policies is close to the throughput achieved by the centralized policies.    	
\end{abstract}
\newcounter{MYtempeqncnt}

\section{Introduction}
Internet-of-things (IoT) \cite{Centenaro_MWC_Oct2016} networks connect a large number of low power sensors whose lifespan is typically limited by the energy that can be stored in their batteries. In this context, the advent of the energy harvesting (EH) technology \cite{Ku_Comm_Tuts_2016} promises to prolong the lifespan of IoT networks by enabling the nodes to operate by harvesting energy from environmental sources, e.g., the sun, the wind, etc. However, this requires the development of new energy management methods. This is because an EH node (EHN) operates under the energy neutrality constraint which requires that the total energy consumed by the node up to any point in time can not exceed the total amount of energy harvested by the node until that point. This constraint is particularly challenging due to the random nature of environmental energy sources. In particular, the evolution over time of the intensity of the sun or wind is a random process, and thus the amount of energy that can be harvested at any given instant can not be deterministically known in advance. In addition, at a given instant, an EHN can only store an amount of energy equal to its battery capacity. Therefore, a major and challenging issue in a EH-based IoT systems is to devise power control policies to maximize the communication performance under the aforementioned constraints. 

Available approaches for power control in EH-based wireless networks can be divided into two main categories: \emph{offline} and \emph{online} approaches. Offline approaches consider a finite time-horizon over which the optimal power control policy has to be designed, and assume that perfect information about the energy arrivals and channel states is available over the entire time-horizon \cite{kaya_TWC_March2012, Wang_JSAC_Mar2015}, before the start of operation. Under these assumptions, the power control problem can be formulated as a static optimization problem aimed at optimizing a given performance metric (e.g. system sum-rate, communication latency), and can be tackled by traditional optimization techniques. However, in general, offline approaches are not practically implementable because they require non-causal knowledge about the energy arrivals and propagation channels. For this reason, offline solutions are mostly considered for benchmarking purposes only. 

In contrast to offline policies, online approaches target the optimization of the system performance over a longer, possibly infinite, time-horizon, and assume that only previous and present energy arrivals and channel states are known \cite{MSharma_TWC_June2018, Baknina_TC_2018}. As a result, the power allocation problem becomes a stochastic control problem, which, upon discretizing the state space (battery state and channel gains), can be formulated as a Markov decision processes (MDP) \cite{Bertesekas-eta-al-2014}, for which optimal long-term policy can be determined numerically. However, these techniques require perfect knowledge of the statistics of the EH process and of the propagation channels, which are difficult to know in practice. In order to address this drawback, the framework of reinforcement learning (RL) \cite{Blasco_TWC_April2013, Ortiz_ICASSP_2018, Nikhilesh_ArXiv_2018, wu_icc_2017, Toorchi_ICIP_sept2016, chu_JIoT_2018, Masadeh_icc_2018, Wei_TWC_Jan2018, xiao_icc_June2015} or that of Lyapunov optimization \cite{Huang_CDC_2015, Yu_arxiv_2018, Gatzianas_TWC_Feb2010} have been proposed to find approximate solutions. All of these previous works take a centralized approach, in which typically the whole network is modeled as a single MDP whose solution provides the optimal long-term power allocation policy for all network nodes. However, this approach is not suitable for large networks, as the presence of a large number of nodes causes inevitable feedback overheads, and more importantly the resulting MDP is often intractable. Indeed, numerical solution techniques for the MDPs suffer from the so-called ``curse-of-dimensionality'' which makes them computationally infeasible.

Therefore, in absence of any a-priori knowledge about the EH process and the channel, it is essential to develop new techniques which can aid in learning the online policies for large EH-based networks, in a \emph{distributed} fashion. A fully distributed approach to online power control will obviate the need for any information exchange between the nodes. Learning distributed power control policies for EH networks have been recently considered in only a handful of works \cite{Miozzo_wcnc_2017, Wang_ICC_2018, Ortiz_ArXiv_2017, Hakami_TVT_June2017}.

In \cite{Miozzo_wcnc_2017}, the authors use a distributed Q-learning algorithm where each node independently learns its individual Q-function. However, the proposed method is not guaranteed to converge, since each individual node experiences an inherently non-stationary environment \cite{Lowe_NIPS2017}. In \cite{Wang_ICC_2018}, a distributed solution is developed to minimize the communication delay in EH-based large networks, assuming the information about the statistics of the EH process and of the wireless channel are known. Interestingly, the interactions among the devices are modeled as a mean-field game (MFG), a framework specifically conceived to analyze the evolution of systems composed of a very large number of distributed decision-makers \cite{Hanif_ACM_Feb2016, Larranaga_isit_june2018, Wang_TWC_MAr2014}. A multi-agent reinforcement learning (MARL) approach is considered in \cite{Ortiz_ArXiv_2017}, where an online policy for sum-rate maximization is developed. However, the approach in \cite{Ortiz_ArXiv_2017} assumes that the global system state is available at each node, which renders it infeasible for large EH networks, due to the extensive signaling required to feedback the global system state to all network nodes. In \cite{Hakami_TVT_June2017}, a two-hop network with EH relays is considered, and a MARL-based algorithm with guaranteed convergence is proposed to minimize the communication delay.

The objective of this work is to develop a mechanism to learn optimal online power control policies in a distributed fashion, for a \emph{fading-impaired} multiple access channel (MAC) with a large number of EH transmitters. The authors in \cite{Wang_JSAC_Mar2015} derived a throughput-optimal \emph{offline} power control policy for a fading EH MAC, which is designed in a centralized fashion. In \cite{Yang_ISIT_2015,Blasco_Mar_JSAC2015}, centralized online policies are developed under the simplifying assumptions of binary transmit power levels, and batteries with infinite or unit-size capacity. Optimal online power control policies for fading EH MAC are not available in the literature, even in a centralized setting. In order to design a centralized power control policy we build upon the recent advances in the deep learning \cite{Zappone_TC_Oct2019}. 
In particular, our main contributions are the following:

\begin{itemize}
	\item  First, to benchmark the performance of the distributed policies, we develop a deep neural network (DNN) based centralized online policy which uses a DNN to map a system state to transmit power. 
	\item We model the problem of throughput maximization for a fading EH MAC as a discrete-time MFG, and exploiting the structure of the problem we show that the MFG has \emph{unique stationary solution}.   
	\item Next, we leverage the fictitious play property of MFGs and develop a deep reinforcement learning (DRL) based approach  to learn the stationary solution of the MFG. Under the proposed scheme, each node apply the DRL, individually, to learn the optimal power control in a completely distributed fashion, without any apriori knowledge about the statistics of the EH process and the channel. 
	\item Furthermore, we adapt the DNN based centralized approach to design an energy efficient distributed online power control policy. 
	\item Extensive numerical results are provided to analyze the performance of the proposed schemes. Our results illustrate that the throughput achieved by DNN based centralized policies is close to the throughput achieved by the offline policies. Moreover, the policies learned using the proposed mean-field MARL approach achieve throughput close to centralized policies. 
\end{itemize}

In contrast to earlier work \cite{Miozzo_wcnc_2017, Wang_ICC_2018}, our algorithm is \emph{provably convergent} and does not require any knowledge about the statistics of the EH process and of the wireless channels. In order to learn the optimal power control policy, each node only needs to know the state of its own channel and battery. The performance of the resulting online policies is very close to offline policies which exploit non-causal information. We note that our work is the first in the literature that uses the multi-agent deep reinforcement learning to obtain the optimal power control in large EH networks. 

The rest of the paper is organized as follows. In the following section, we describe the system model and the problem formulation. In Sec.~\ref{sec:SOlution_Method}, we design DNN based centralized online power control policies. Next, In Secs.~\ref{Sec:mean_field_MARL} and \ref{sec:MF_MARL_Algo}, we model the throughput maximization problem as a discrete-time finite state MFG and present our mean-field MARL approach to learn the distributed power control policy, respectively. In Sec.~\ref{sec:energy_compare}, we analyze the energy cost incurred on the implementation of the proposed algorithms, and also propose an energy efficient distributed DNN based algorithm. Simulation results are presented in Sec.~\ref{sec:sim}, and conclusions in Sec~\ref{sec:Concl}. 

\section{System Model and Problem Formulation}
\label{sec:sys}
We consider a time-slotted EH network, where a large number of \emph{identical} EHNs transmit their data over block fading channels to an access point (AP) which is connected to the mains. The set of transmitters is denoted by $\mathcal{K}\triangleq\{1,2,\ldots,K\}$, where $K\gg 1$ denotes the number of EHNs. 
In the $n^{\text{th}}$ slot, the \emph{fading} complex channel gain between the $k^{\text{th}}$ transmitter and the AP is denoted\footnote{For any symbol in the paper, the superscript and subscript represent the node index and the slot index, respectively, and if only the subscript is present then it denotes either the node index or the slot index, depending on the context.} by $g_n^k\in \mathbf{G}_k$. In each slot, the channel between any transmitter and the AP remains constant for the entire slot duration, and changes at the end of the slot, independently of the channel in the previous slot. We assume that the  wireless channels between the nodes and the AP, $\mathbf{G}_k$, are identically distributed. 
\begin{figure}[t!]
	\centering
	\includegraphics[width=2.5in]{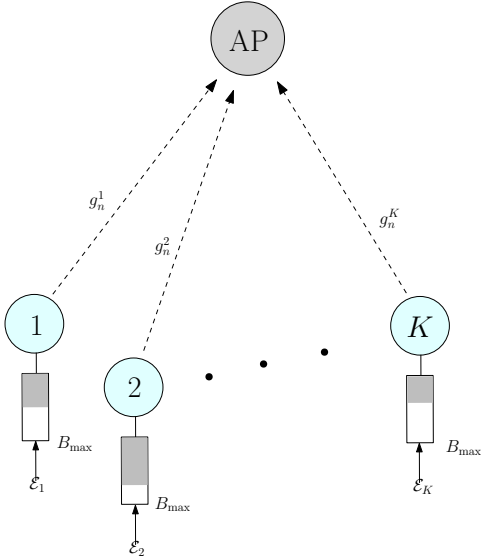}
	\caption{System model for the fading impaired EH multiple access network. The EH process and battery size at the $k^{\text{th}}$ node are denoted by $\mathcal{E}_k$ and $B_{\max}$, respectively. The nodes transmit their data to the AP over block fading channels. The gain of the fading complex channel from the transmitter $k$ to the AP, in the $n^{\text{th}}$ slot, is denoted by $g_n^k$.}
	\label{Fig:sys_mod}
\end{figure}

In a slot, the $k^{\text{th}}$ node harvests energy according to a general stationary and ergodic harvesting process  $f_{\mathcal{E}_k }(e_k)$, where the random variable $\mathcal{E}_k$ denotes the amount of energy harvested by the $k^{\text{th}}$ transmitter, and $e_k$ denotes a realization of $\mathcal{E}_k$. We assume that the harvesting processes $\{\mathcal{E}_k\}_{k\in\mathcal{K}}$ are identically distributed across the individual nodes, but not necessarily independent of each other. At each node, the harvested energy is stored in a perfectly efficient, finite capacity battery of size $B_{\max}$. Further, only \emph{causal} and \emph{local} information is available, i.e., each node knows only its \emph{own} energy arrivals, battery states, and the channel states to the AP, in the current and all the previous time slots. In particular, no node has information about the battery and the channel state of the other nodes in the network. Also, at any node, no information is available about the distribution of the EH process and of the wireless channels. 

Let $p_n^k\leq P_{\max}$ denote the transmit energy used by the $k^{\text{th}}$ transmitter in the $n^{\text{th}}$ slot, where $P_{\max}$ denotes the maximum transmit energy which is determined by the RF front end of the EHNs. Further, $\mathcal{P}_n\triangleq\{p_n^k\}_{k=1}^K$
denotes the vector of transmit energies used in the $n^{\text{th}}$ slot, by all the transmitters. The battery at the $k^{\text{th}}$ node evolves as  
\be
B_{n+1}^k=\min\{[B_{n}^k+e_n^k-p_n^k]^+, B_{\max}\},
\label{eq:battery_evol}
\ee
where $1\leq k\leq K$, and $[x]^+\triangleq\max\{0,x\}$. In the above, $B_{n}^k$ and $e_n^k$ denote the battery level and the energy harvested by the $k^{\text{th}}$ node at the start of the $n^{\text{th}}$ slot, respectively. 
An upper bound on the successful transmission
rate of the EH MAC over $N$ slots is given by\cite{Wang_JSAC_Mar2015}
\be
\mathcal{T}(\mathcal{P}) = \sum_{n=1}^N\log\left(1+\sum_{k\in\mathcal{K}}p_n^kg_n^k\right),
\label{eq:throughput_finite}
\ee
where $\mathcal{P}\triangleq\{\mathcal{P}_n|1\leq n\leq N\}$. Note that, the above upper bound can be achieved by transmitting independent and identically distributed (i.i.d.) Gaussian signals. In \eqref{eq:throughput_finite}, for simplicity, and without loss of generality, we set the power spectral density of the AWGN at the receiver as unity. \footnote{We note that, in a scenario when all the EHNs simultaneously transmit their data, the cumulative signal-to-noise ratio (SNR) term  in \eqref{eq:throughput_finite}, $\sum_{k\in\mathcal{K}}p_n^kg_n^k$, grows with the number of users in the network. In practice, this problem can be circumvented by ensuring that the transmit power of EHNs scales down in inverse proportion to the number of users, i.e., $O\left(\frac{1}{K}\right)$. This ensures that the total energy in the network stays finite.}   

In the absence of information about the statistics of the EH process and the channel, our goal in this work is to learn online energy management policy at each node to maximize the time-averaged sum throughput. The optimization problem can be expressed as follows
\beseq
\begin{align}
&\max_{\{\mathcal{P}\} } \liminf_{N\to \infty}\frac{1}{N}\mathcal{T}(\mathcal{P}),\\
&\text{s.t. }  0\leq  p_n^k\leq \min\{B_n^k,P_{\max}\}, 
\label{eq:optim_prob_const}
\end{align}
\label{eq:optim_prob}
\eeseq
for all $n$ and $1\leq k\leq K$. Constraint \eqref{eq:optim_prob_const} captures the fact that the maximum energy a node can use in the $n^{\text{th}}$ slot is limited by the minimum between the amount of energy available in the battery, $B_n^k$, and the maximum allowed transmit energy $P_{\max}$. Note that, the information about the \emph{random} energy arrivals and the channel is only \emph{causally available}, and for each node the battery evolves in a Markovian fashion, according to \eqref{eq:battery_evol}. Hence, the optimization problem \eqref{eq:optim_prob} is essentially a stochastic control problem which, upon discretization of the state space, could be formulated as a Markov decision process (MDP). However, solving such an MDP in the considered setting poses at least three major challenges:
\begin{itemize}
	\item Infeasible complexity, since in the considered setup a large number of nodes, $K$, is present in the network.
	\item In each slot, the global information about the battery and channel states, and the value of the harvested energy of \emph{each} network node would be needed for the operation of the policy. Therefore, the feedback overhead in \emph{each} slot is $\mathcal{O}(K)$. For a network with large number of nodes this would result in a significant control overhead.  
	\item Finally, solving the MDP also requires statistical information about the EH process and the wireless channel, which is often difficult to obtain, and indeed is not assumed in this work. 
\end{itemize}
Due to these reasons, the goal of this work is to develop a framework to learn online power control policies in a distributed fashion, i.e., each node learns the optimal online power control policy without requiring to know the battery and channel states, and actions of the other nodes. In the following sections, we develop a provably convergent mean-field  multi-agent reinforcement learning (MF-MARL) approach to distributively learn the throughput-optimal power control policies, leveraging the tools of DRL and MFGs. In the following section, first we present a DNN based centralized power control policy which is used for benchmarking our MF-MARL based distributed solution.

\section{DNN based Centralized Online Power Control Policy}
\label{sec:SOlution_Method}

To describe our DNN based centralized approach to solve the stochastic control problem in \eqref{eq:optim_prob}, we define some additional notations and formally define the online and offline policies in the context of our problem.
\subsection{Notations} For the $k^{\text{th}}$ node, let ${\boldsymbol E}_{m:n}^k\triangleq\{e_m^k, e_{m+1}^k,\ldots,e_n^k\}$, ${\boldsymbol B}_{m:n}^k\triangleq\{B_m^k, B_{m+1}^k,\ldots,B_n^k\}$, and ${\boldsymbol G}_{m:n}^k\triangleq\{g_m^k, g_{m+1}^k,\ldots,g_n^k\}$ denote the vectors containing the values of energy harvested, battery state, and the channel state, respectively, in the slots from $m$ to $n$. Further, history up to the start of slot $n$ is denoted by a tuple $\boldsymbol{H}_n\triangleq\left\{({\boldsymbol E}_{1:n-1}^k, {\boldsymbol B}_{1:n-1}^k, {\boldsymbol G}_{1:n-1}^k)\right\}_{k=1}^K$, where $\boldsymbol{H}_n \in \mathcal{H}_n$, where $\mathcal{H}_n$ is the set of all possible histories up to slot $n$. Also, in the $n^{\text{th}}$ slot the current state of the system is described by the tuple $\boldsymbol{s}_n\triangleq\{\boldsymbol{E}_n,\boldsymbol{B}_n,\boldsymbol{G}_n\}$, where $\boldsymbol{E_n}\triangleq(e_n^1,e_n^2,\ldots,e_n^K)$, $\boldsymbol{B_n}\triangleq(B_n^1,B_n^2,\ldots,B_n^K)$ and $\boldsymbol{G_n}\triangleq(g_n^1,g_n^2,\ldots,g_n^K)$ are the vectors containing the values of energy harvested, battery state, and the channel state, respectively, for all the nodes in the $n^{\text{th}}$ slot. Further, $\boldsymbol{s}_n\in \mathcal{S}$ where $\mathcal{S}$ denotes the set of all the possible states.  

\subsection{Online and Offline Policies}
In the $n^{\text{th}}$ slot, an online decision rule $f_n:\mathcal{H}_n \times \mathcal{S}\to \hat{\mathcal{P}}$ maps the history, $\boldsymbol{H}_n$, and the current state of the system, $\boldsymbol{s}_n$, to a transmit energy vector $\hat{\mathcal{P}}\in \mathbb{R}_+^K$ which contains feasible transmit energies for all the nodes. Mathematically, an \emph{online} policy $\mathcal{F}$ is the collection of decision rules, i.e., $\mathcal{F}\triangleq\{f_1,f_2\ldots\}$. 
In contrast, for \emph{offline} policy design problem the time-horizon, $N$, is finite, and, for all the slots, the information about the amount of the energy harvested and the channel state is available non-causally, i.e., before the start of the operation, for all the slots. Hence, the stochastic control problem in \eqref{eq:optim_prob} reduces to a static optimization problem which is written as
\beseq
\begin{align}
&\hspace{45pt}\max_{\{\mathcal{P}\} } \frac{1}{N}\sum_{n=1}^N\log\left(1+\sum_{k\in\mathcal{K}}p_n^kg_n^k\right),\\
\text{s.t. } & 0\leq  p_n^k\leq \min\{B_n^k,P_{\max}\}  \text{ for all } n, \text{ and } 1\leq k\leq K.
\label{eq:optim_prob_const_offline}
\end{align}
\label{eq:optim_prob_offline}
\eeseq
Note that, since $N$ is finite, and the realizations of the EH processes and the channel states are known non-causally, i.e., $E_{1:N}^k$ and $G_{1:N}^k$ are known at the start of the operation, for all the nodes, the objective and constraints in \eqref{eq:optim_prob_offline} are deterministic convex functions in the optimization variables $p_k^n$. Hence, the  offline policy design problem in \eqref{eq:optim_prob_offline} is a convex optimization problem which can be solved efficiently using the iterative algorithm presented in \cite{Wang_JSAC_Mar2015}, with per iteration complexity equal to $\mathcal{O}\left(KN^2\right)$. The following section presents our approach to obtain the DNN based online energy management policies which, in general, can also be used for solving a stochastic control problem using the solution of an offline optimization problem.  
\begin{figure}[t!]
	\centering
	\includegraphics[width=3in]{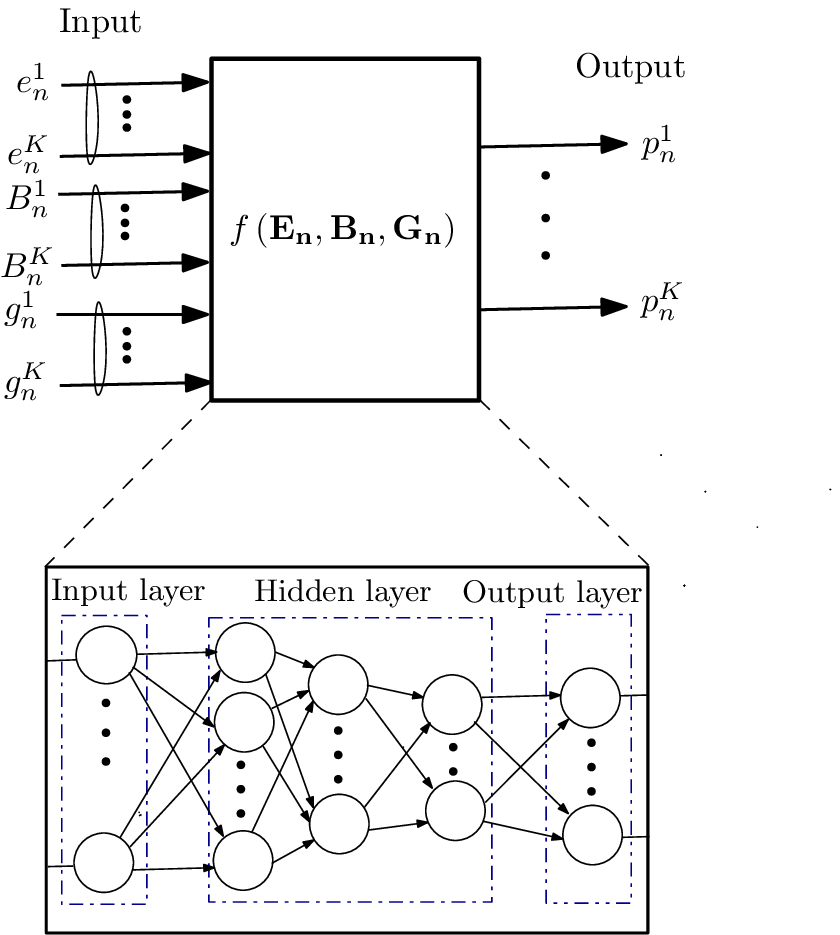}
	\caption{DNN based centralized online power control. In the $n^{\text{th}}$ slot, the DNN maps the system state $(\mathbf{E_n,B_n,G_n})$ to a feasible transmit energy vector $\hat{\mathcal{P}}$ which contains the transmit energies of all the transmitters of the EH MAC. Here, the system state is denoted by the collection of the values of the energy harvested by all the nodes, $\boldsymbol{E_n}\triangleq\{e_n^1,e_n^2,\ldots,e_n^K\}$, battery states of all the nodes, $\boldsymbol{B_n}\triangleq\{B_n^1, B_n^2,\ldots,B_n^K\}$, and the state of channels between all the transmitters and the AP, $\boldsymbol{G_n}\triangleq\{g_n^1, g_n^2,\ldots,g_n^K\}$. Thus, the DNN takes $3K$ inputs which are then processed by $L$ hidden layers and one output layer to output a $K$ length feasible transmit power vector.  }
	\label{Fig:DNN_arch}
\end{figure}

\subsection{DNN based Online Energy Management} 
To obtain online energy management policy, we first note that due to finite state and action space of the problem, the optimal policy for the problem \eqref{eq:optim_prob} is a Markov deterministic policy\cite[Thm. 8.4.7]{Puterman_MDP_2014}, i.e., $\mathcal{F}\triangleq\{f,f\ldots\}$ where $f:\mathcal{S}\to \hat{\mathcal{P}}$. Hence, the optimal online energy management policy can be obtained by finding a decision rule which maps the current state of the system $\boldsymbol{s_n}$ to an optimal transmit energy vector for problem \eqref{eq:optim_prob}. Furthermore, for a finite horizon problem, an offline policy also represents a mapping from the current state to a feasible transmit energy vector, i.e., the optimal offline policy maps a $(\boldsymbol{E},\boldsymbol{B}, \boldsymbol{G})$ tuple to $\hat{\mathcal{P}^*}$. Here, $\hat{\mathcal{P}^*}$ denotes the vector containing the optimal transmit power for \emph{each} node. Since, a DNN is a universal function approximator\cite{Goodfellow-et-al-2016}, provided it contains a sufficient number of neurons, we propose to use a DNN to learn the optimal decision rule by using the solution of the offlline policy design problem to train the DNN. Under the proposed online scheme, in a given slot, the optimal transmit energy vector can be obtained by feeding the current state of the system as the input to the trained DNN. Our approach is illustrated in Fig.~\ref{Fig:DNN_arch}. In the following, we briefly describe the architecture of the DNN used and the procedure used for training the DNN.

\subsection{DNN Architecture}  
We adopt a feedforward neural network whose input layer contains $3K$ neurons, one corresponding to each input. A $3K$-length vector, containing the states of all the transmitters, is fed to the DNN as input which is then processed by $h+1$ layers ($h$ hidden layers and the output layer) to compute a feasible $K$-length transmit power vector. The number of processing units, usually termed as neurons, at the $j^{\text{th}}$ layer is denoted by $N_j$, where $1\leq j\leq h+2$. Note that, $N_1=3K$ and $N_{h+2}=K$. The output of the $n^{\text{th}}$ neuron of the $j^{\text{th}}$ layer, denoted by $I_j(n)$, is computed as 
\be
I_j(n) =  F_{j,n}\left(\boldsymbol{W}_{j,n}^TI_{j-1}+b_{j,n}\right),
\ee  
where $I_{j-1}$ is the output of the $(j-1)^{\text{th}}$ layer, which is fed as input to the $j^{\text{th}}$ layer. Also, $\boldsymbol{W}_{j,n}\in \mathbb{R}^{N_{j-1}}$, $b_{j,n}\in \mathbb{R}$, and $F_{j,n}$ denote the weights, bias and the nonlinear activation function for the $n^{\text{th}}$ neuron of $j^{\text{th}}$ layer, respectively. For detailed exposition on the architecture of DNNs and activation functions we refer the readers to \cite{Goodfellow-et-al-2016}. 
\subsection{Training}
The DNN can learn the optimal mapping, between the system state and the feasible transmit power vector, by appropriately adjusting the weights and biases of the neurons in the network. The weights $\boldsymbol{W} = \{\{\boldsymbol{W}_{j,n}\}_{n=1}^{N_j}\}_{j=1}^{h+2}$and biases $\boldsymbol{b}=\{\{b_{j,n}\}_{n=1}^{N_j}\}_{j=1}^{h+2}$ of the neurons of a DNN can be tuned by minimizing a loss function over a training set which is a set of data points for which the optimal mapping is already known. In particular, the training process minimizes the average loss, over the entire training set, defined as follows
\be
L_{\text{av}}(\boldsymbol{W}, \boldsymbol{b})=\frac{1}{N_{\text{data}}}\sum_{\ell=1}^{N_{\text{data}}}L\left(\hat{\mathcal{P}_{\ell}^*},I_{h+2, \ell}(\boldsymbol{W},\boldsymbol{b})\right),
\label{eq:loss_function}
\ee   
where $L(\cdot)$ denotes a loss function which is a metric of distance between the desired output and the output of the DNN, and $N_{\text{data}}$ denotes the number of data points in the training set. In \eqref{eq:loss_function}, $\mathcal{P}_{\ell}^*$ and $I_{h+2, \ell}(\cdot,\cdot)$ denote the actual output and the output of the DNN, respectively, for the $\ell^{\text{th}}$ data point. The training proceeds by iteratively minimizing the loss in  \eqref{eq:loss_function}, using gradient based methods over the training data set. During the training, the gradients are often estimated using small subsets of the training set which are called as mini-batches. 
	Note that, in order to train the DNN to learn the optimal online energy management decision rule, the training data is generated by solving several instantiations of the offline problem \eqref{eq:optim_prob_offline}, each corresponding to a different realization of $\{\boldsymbol{E}_{1:N}^k, \boldsymbol{G}_{1:N}^k\}_{k=1}^K$. The training data generated by solving the offline problem contains the tuples of the form $\{\left(\boldsymbol{E},\boldsymbol{B},\boldsymbol{G}\right),\boldsymbol{P}\}$, where $\left(\boldsymbol{E},\boldsymbol{B},\boldsymbol{G}\right)$ and $\boldsymbol{P}$ represents the input to the DNN and the desired output, respectively. Further details related to the loss function, training method, and the batch size used in this work are presented in Sec.~\ref{sec:sim}. A detailed discussion on the choice of the loss functions for the training, the training method for DNN, and the mini-batch size can be found in \cite[Ch. 7 and 8]{Goodfellow-et-al-2016}. 

Note that, our approach to design centralized online energy management policy does not require the knowledge about the statistics of the EH process and the channel. Interestingly, as observed through the simulations, the proposed DNN based approach performs marginally \emph{better} than the state-of-the-art deep  reinforcement learning approach. However, in contrast to deep Q-learning method, the proposed DNN-based approach requires the measurements of EH values and channels for all the nodes, which is used for training the DNN before the start of the operation. Also, the proposed approach determines the transmit energy vector for all the nodes in a centralized fashion, using the battery state, channel state, and the amount of energy harvested in the current slot, for all the nodes. To implement this scheme, the nodes are required to feedback their state in every slot and then the transmit energies to be used in the next slot are communicated to the nodes. The distributed solutions proposed in the following sections obviate the overhead involved in communication of the state information and the transmit energies.

\section{Mean-field Game to Maximize the Sum Throughput}
\label{Sec:mean_field_MARL}
In this section, first we model the sum throughput maximization problem in \eqref{eq:optim_prob} as a discrete time, finite state MFG\cite{GOMES_JMPA_2010}. Next, we present preliminaries on the discrete-time MFGs, and list the key results which are useful in showing the convergence of the proposed approach to the stationary solution of the MFG. 
\subsection{Throughput Maximization Game}
The throughput maximization game $\mathcal{G}_T\triangleq \{\mathcal{K},  \mathcal{S}, \mathcal{F}, \mathcal{R}\}$ consists of: 
\begin{itemize}
	\item The set of players $\mathcal{K}= \{1,2,\ldots,K\}$, each one corresponding to a unique EH transmitter, where $K>>1$;
	\item The state space of all players $\mathcal{S}\triangleq \times_{k\in\mathcal{K}} \mathcal{S}^k$, with $\mathcal{S}^k$ denoting the space of all the states $s^k$ for the $k^{\text{th}}$ transmitter, and  $|\mathcal{S}^k| \triangleq d$. Also, let $s_n^k\triangleq (B_n^k, g_n^k, e_n^k)$ denote the state of the $k^{\text{th}}$ transmitter in the $n^{\text{th}}$ slot, where $B_n^k$, $g_n^k$, and $e_n^k$ are discrete-valued; 
	\item The set of energy management policies of all the nodes $\mathcal{F}\triangleq\{\mathcal{F}^k\}_{k\in\mathcal{K}}$, where $\mathcal{F}^k$ denotes the policy of the $k^{\text{th}}$ node; 
	\item The set of reward functions of all the nodes $\mathcal{R}\triangleq \{\mathcal{R}_k\}_{k\in\mathcal{K}}$, where $\mathcal{R}_k$ is the reward function of node $k$.
\end{itemize}
Note that, since all the transmitters are identical, the state space of individual nodes, $\mathcal{S}^k$, is the same set for all $k=1,\dots,K$. In the $n^{\text{th}}$ time slot, the $k^{\text{th}}$ node uses $p_n^k$ amount of energy, prescribed by its policy $\mathcal{F}^k$, and collects a reward according to its reward function $\mathcal{R}_k$ and evolves from one state to another.

Under the mean field hypothesis\cite{GOMES_JMPA_2010}, the reward obtained by a given node depends on the other nodes only through the distribution of all the nodes across the states. Let $\boldsymbol \pi_n\triangleq (\pi_n^1,\ldots,\pi_n^d)$ denote the distribution of all the nodes across the states, in the $n^{\text{th}}$ slot, where $\pi_n^i$ denotes the fraction of nodes in the $i^{\text{th}}$ state. Since the goal is to maximize the sum-throughput of the network, each node receives a reward equal to the sum-throughput of the network. In the $n^{\text{th}}$ slot, the reward obtained by the $k^{\text{th}}$  node is equal to the total number of bits successfully received by the AP, from all the transmitters. Thus, the reward function can be mathematically expressed as
\begin{align}
\mathcal{R}_k(\boldsymbol \pi_n, p_n^k) &\triangleq \log\left(1+p_n^k+\sum_{i=1}^d (K-1)\pi_n^ip_ig_i\right)\nonumber\\
&=\log\left(1+\sum_{i=1}^d K\pi_n^ip_ig_i\right),
\label{eq:reward} 
\end{align}   
where $g_i$ is the wireless channel gain between the nodes in the $i^\text{th}$ state and the AP, and $p_i\in\mathcal{A}_p\triangleq\{0,p_{\min},\ldots, P_{\max} \}$ denotes the energy level used for transmission by the nodes in the $i^{\text{th}}$ state. Here, $p_{\min}$ denotes the minimum energy required for transmission. In \eqref{eq:reward}, $K\pi_n^i$ denotes the fraction of nodes in the $i^{\text{th}}$ state, in the $n^{\text{th}}$ slot. Note that, \eqref{eq:reward}  is written using the fact that under the mean-field hypothesis all the nodes are identical, and hence use the same policy, which also implies that the reward function, $\mathcal{R}_k(\cdot,\cdot)$, is identical for all the nodes. Hence, to simplify the notations, in the ensuing discussion we omit the node index $k$. Also, \eqref{eq:reward} implicitly assumes that all nodes in state $i$ use the energy $p_i$ which is essentially motivated by the fact that for an MDP with finite state and action sets, the optimal policy is a Markov deterministic policy\cite[Thm. 8.4.7]{Puterman_MDP_2014}, i.e., in a slot the optimal transmit energy for a node depends only on its current state. 

In the $n^{\text{th}}$ slot, when a node in state $s_n\in \mathcal{S}$ transmits using energy $p_{s_n}$, the system evolves as
\begin{align}	
\label{eq:dsitribution_evolution}
\pi_{n+1}^j&= \sum_{i}\pi_n^iP_{ij}^n\left(p_i\right),
\end{align}
where $P_{ij}^n(\cdot)$ denotes the probability in the slot $n$ that a node in state $i$ transits to state $j$, and depends on,  $p_i$, the energy used for transmission by the node in the $i^{\text{th}}$ state\footnote{In a general MFG the transition probabilities $P_{ij}^n$ may also depend on the actions of the other players.}. In \eqref{eq:dsitribution_evolution}, the transition probabilities,  $P_{ij}^n(\cdot)$, are determined by the statistics of the EH process and the wireless channel, and the transmit power policy used by a node\footnote{Thus, if a node follows a transmit power policy which evolves over the time, the resulting transition probabilities are non-homogeneous over time. }. In a given slot, all the nodes obtain a reward, $\mathcal{R}\left(\boldsymbol \pi_n, \mathcal{F}\right)$, equal to the total number of bits successfully decoded in that slot, by the AP.

For a given node, starting from the $n^{\text{th}}$ slot, the expected sum-throughput obtained by following a policy $\mathcal{F}$ can be expressed as
\begin{align}
V_n(\boldsymbol\pi_n,\mathcal{F}) =  \mathcal{R}\left(\boldsymbol\pi_n,\mathcal{F}\right)+V_{n+1}\left(\boldsymbol\pi_{n+1},\mathcal{F}\right),
\label{eq:value_function}
\end{align} 
where $V_{n+1}(\boldsymbol\pi_{n+1},\mathcal{F})$ denotes the expected throughput obtained by following a policy $\mathcal{F}$ starting from slot $n+1$, when in the $(n+1)^{\text{th}}$ slot the distribution of the nodes across the states is given by $\boldsymbol\pi_{n+1}$. In the rest of the paper $V(\cdot,\cdot)$ is termed as the value function. In the above, similar to an MDP \cite{Puterman_MDP_2014}, \eqref{eq:value_function} is written using the fact that the expected sum-throughput obtained by following a policy $\mathcal{F}$, starting from the time slot $n$, is equal to the sum of the expected sum-throughput obtained in the slot $n$ and the slot $n+1$ onward. Note that, under the mean-field hypothesis, the expected sum-throughput in \eqref{eq:value_function} is identical for all the nodes, and due to special structure of the reward function, the value function of each node, $V(\cdot,\cdot)$, only depends on the distribution of the nodes across the states, $\boldsymbol\pi_n$, not on the state of the individual nodes. Hence, \eqref{eq:value_function} does not include a superscript/subscript to denote the node index. In the following, we present preliminaries on discrete-time, finite state MFGs. 

\subsection{Preliminaries: discrete-time finite state MFGs}      
In the following, we define the notion of Nash equilibrium and stationary solution for the discrete-time MFGs, and briefly summarize the key results used to prove the convergence of the proposed MARL algorithm in Sec.~\ref{sec:MF_MARL_Algo}. For a detailed exposition on discrete-time finite state MFGs we refer the readers to~\cite{GOMES_JMPA_2010}.
\begin{defn}[Nash maximizer]
	For a fixed probability vector $\boldsymbol\pi_n$, a policy $\mathcal{F}^*$ is said to be a Nash maximizer if and only if
	$$V_n(\boldsymbol\pi_n,\mathcal{F})\leq V_n(\boldsymbol\pi_n,\mathcal{F}^*), \text{ for all policies } \mathcal{F}.$$
	
\end{defn}
That is, for a fixed $\boldsymbol\pi_n$, the Nash maximizer is a policy that maximizes the value function. Next, for a discrete-time finite state MFG, we define the notions of solution and stationary solution.
\begin{defn}[Solution of a MFG]
	Suppose that for each $\boldsymbol\pi_n$ there exists a Nash maximizer $\mathcal{F}^*$. Then a sequence of tuples $\{(\boldsymbol\pi_n,V_n) \text{ for } n\in \mathbb{N}\}$ is a solution of the MFG if for each $n\in \mathbb{N}$ it satisfies \eqref{eq:dsitribution_evolution} and \eqref{eq:value_function} for some Nash maximizer of $V_n$.
\end{defn}
\begin{defn}[Stationary solution]
	Let $\mathcal{G}_{\boldsymbol\pi}$ and $\mathcal{K}_V$ be defined as $\mathcal{G}_{\boldsymbol\pi_n}(V_{n+1}) = V_n(\boldsymbol \pi_n, \mathcal{F}),$ and 
	$\mathcal{K}_{V_n}(\boldsymbol \pi_{n}) = \boldsymbol \pi_{n+1}$. A pair of tuple $(\tilde{\boldsymbol\pi},\tilde{V})$ is said to be a stationary solution if and only if 
	\begin{align}
	\mathcal{G}_{\tilde{\boldsymbol\pi}}(\tilde{V}) = \tilde{V}\text{ and }
	\label{eq:bellman_operator}	\\
	\mathcal{K}_{\tilde{V}}(\tilde{\boldsymbol\pi}) = \tilde{\boldsymbol\pi}.
	\label{eq:FP_operator}
	\end{align}
\end{defn}

Note that, the operators $\mathcal{K}_{V_n}(\cdot)$ and $\mathcal{G}_{\boldsymbol\pi_n}(\cdot)$ are backward and forward in time, respectively. Also, the operators in \eqref{eq:bellman_operator} and \eqref{eq:FP_operator} are compact representations of \eqref{eq:dsitribution_evolution} and \eqref{eq:value_function}, respectively. The stationary solution of a MFG, $(\tilde{\boldsymbol\pi},\tilde{V})$, is a fixed-point of operators $\mathcal{G}_{\boldsymbol \pi}$ and $\mathcal{K}_{V}$ which are essentially discrete time counterparts of Hamilton-Jacobi-Bellman and Fokker-Planck equations. Next, we list the results which identify the conditions under which a stationary solution exists. We omit the proofs for brevity. These results are later used for proving the convergence of our mean-field MARL (MF-MARL) algorithm to the stationary solution. 

\begin{thm}[Uniqueness of Nash maximizer (Theorem 2 \cite{GOMES_JMPA_2010})]
	Let $f_i(p_i) \triangleq\frac{\partial V(\boldsymbol \pi,\mathcal{F})}{\partial p_i}$ where $p_i\in\left[0, P_{\max}\right]$ for all $1\leq i\leq d$. If the value function $V_n$ is convex and continuous with respect to $p_i$, and $f_i$ is strictly diagonally convex, i.e., it satisfies
	\begin{align}
	\sum_{i=1}^d(p_i^1- p_i^2)(f_i(\mathcal{F}^1)-f_i(\mathcal{F}^2))>0,
	\end{align}
	then there exists a unique policy which is a Nash maximizer for the value function $V$. Here, $p_i^1$ and $p_i^2$ denote the actions prescribed in the $i^{\text{th}}$ state by two arbitrary policies $\mathcal{F}^1$ and $\mathcal{F}^2$, respectively. 
	\label{thm:unique_Nash_maximizer}
\end{thm}
The following result shows that if the reward function is monotonic with respect to both the variables, $\boldsymbol\pi$ and $p_i$, then the MFG admits a unique solution.  
\begin{thm}[Uniqueness of solution (Proposition 4.3.1, \cite{SHadikhanloo_PhDTheis_Jan2018})]
	Let the value function be a continuous function with respect to both of its arguments, and also assume that there exists a unique Nash maximizer $\mathcal{F}_n$ for all $n\in\{0,1,2,\cdots\}$. Further, let the reward function be monotone with respect to the distribution $\boldsymbol\pi$, i.e,
	\begin{align}
	\sum_{i=1}^d(\pi_i^2- \pi_i^1)(\mathcal{R}_i(\mathcal{F},\pi^2)-\mathcal{R}_i(\mathcal{F},\pi^1))\geq 0,
	\label{eq:reward_function_monotonicity_dist}
	\end{align}
	then there exists a unique solution for the MFG. In the above $\mathcal{R}_i(\cdot,\cdot)$ denotes the reward obtained by the nodes in the $i^{\text{th}}$ state.
	\label{thm:unique_stationary_solution}
\end{thm}
In addition, the uniqueness of the Nash maximizer and the continuity of the value function in both of its arguments ensure that a stationary solution exists \cite[Thm. 3]{GOMES_JMPA_2010}. Thus, Theorem~\ref{thm:unique_stationary_solution} also implies that the stationary solution is unique. In the following, we establish that the MFG $\mathcal{G}_{T}$ admits a unique stationary solution.
\subsection{Unique Stationary Solution for $\mathcal{G}_T$}
\begin{thm} The throughput maximization mean-field game $\mathcal{G}_T$ has a unique solution.
	\label{thm:unique_stationary_soln_GT}
\end{thm}
\begin{proof}
	Proof is relegated to Appendix~\ref{app:proof_thm_unique_stationary_soln_GT}
\end{proof}

The uniqueness of the solution of a discrete-time MFG implies that if an algorithm learning the solution of the game converges, then it converges to the unique stationary solution. In the next section, we present an algorithm to learn the stationary solution of the MFG $\mathcal{G}_T$ as well as the corresponding Nash maximizer power control policy and provide convergence guarantees for it. The proposed approach is termed as MF-MARL approach, as in this approach each individual node uses the reinforcement learning technique, to learn the stationary solution of the game. 

\section{\textcolor{black}{MF-MARL for Distributed Power Control}}
\label{sec:MF_MARL_Algo}
In this section, we present our mean-field MARL approach to learn the online power control policies to maximize the throughput of a fading EH MAC with large number of users. We show that the proposed approach enables the distributed learning of the power control policies which eventually converge to the \emph{stationary} Nash equilibrium. The proposed MF-MARL algorithm exploits the fact that discrete time finite state MFGs have the fictitious play property (FPP) \cite{SHadikhanloo_PhDTheis_Jan2018}. The FPP for a discrete time MFG is described in the following. Let $m$ denote the iteration index and $\boldsymbol{\bar\pi}_1$ denote an arbitrary probability vector representing the initial distribution of the nodes across the states. Let

\begin{align}
\mathcal{F}_m^* &\triangleq \arg \max_{\mathcal{F}} V_m\left(\boldsymbol{\bar\pi}_m,\mathcal{F}\right),\label{eq:FPP1}\\
\boldsymbol{\pi}_{m+1} &= \mathcal{K}_{V_m(\mathcal{F}_m^*)}(\boldsymbol{\pi}_{m}),\label{eq:FPP2}\\
\text{ and } \boldsymbol{\bar \pi}_{m+1} &=\frac{m}{m+1}\boldsymbol{\bar\pi}_m+\frac{1}{m+1}\boldsymbol{\pi}_{m+1}\label{eq:FPP3}.
\end{align}
The procedure described by \eqref{eq:FPP1}, \eqref{eq:FPP2} and \eqref{eq:FPP3} is called the fictitious play procedure. \textcolor{black}{As described in \eqref{eq:FPP1}, at the $m^{\text{th}}$ iteration, a node attempts to learn the Nash maximizer, $\mathcal{F}_m^*$, given that its belief about the distribution of the nodes across the states is $\boldsymbol{\bar\pi}_m$. Based on the Nash maximizer learned at the $m^{\text{th}}$ iteration, $\mathcal{F}_m^*$, the belief about the distribution is updated to $\boldsymbol{\bar\pi}_{m+1}$, using \eqref{eq:FPP2} and \eqref{eq:FPP3}. Next, at the $(m+1)^{\text{th}}$ iteration, each node attempts to learn the Nash maximizer, $\mathcal{F}_{m+1}^*$.} A discrete-time MFG is said to have FPP  if and only if the procedure described by \eqref{eq:FPP1}, \eqref{eq:FPP2} and \eqref{eq:FPP3} \emph{converges}. The following result provides the conditions under which the fictitious play procedure converges to the unique stationary solution of the discrete-time MFG.    
\begin{thm}[Convergence of FPP to unique stationary solution (Theorem 4.3.2 \cite{SHadikhanloo_PhDTheis_Jan2018})]
	Let $(\boldsymbol\pi_m,V_m)$ denote the sequence generated through the FPP. If a MFG has a unique Nash maximizer at each stage of the game and the reward function is continuous and monotone with respect to probability vector $\boldsymbol\pi$ then the sequence $(\boldsymbol\pi_m,V_m)$ converges to $(\tilde{\boldsymbol\pi},\tilde{V})$ the unique stationary solution of the MFG.
	\label{thm:FPP_convergence}
\end{thm}
For the throughput maximizing MFG $\mathcal{G}_T$, convergence of the FPP to the stationary solution of the game directly follows from the above result and Theorem~\ref{thm:unique_stationary_soln_GT}. As a consequence of this result, the stationary solution of the MFG $\mathcal{G}_T$ can be learned through the fictitious play procedure, provided  the Nash maximizer can be found at each iteration of the fictitious play procedure, and the belief about the distribution is updated correspondingly. The MF-MARL proposes to use the reinforcement learning to learn the Nash maximizer at each iteration, i.e., for a given belief distribution $\bar{\boldsymbol \pi}$ each node \textcolor{black}{individually} uses a reinforcement learning algorithm to learn the Nash maximizer. The proposed MF-MARL approach is described in Algorithm~\ref{algorithm:MF_MARL}. 
\begin{algorithm}[t]
	\caption{: MF-MARL approach to learn online policies}
	
	\begin{algorithmic}
		\State {\bf Initialize}: ${\boldsymbol{\bar \pi_1}}$ to a valid probability vector, $\epsilon_1$, $\tilde{\epsilon}$, $T$ and $m\gets 0,n\gets 0$
		\Do
		\setdefaultleftmargin{10cm}{0.5cm}{}{}{}{}
		\ben
		\item \textcolor{black}{Set $m\gets m+1$}; \textcolor{black}{during each iteration}, at each node execute Q-learning algorithm to learn Nash maximizer ${\mathcal{F}_m^k}^*$
		\item In the $n^{\text{th}}$ time-slot, $n\leq T$, \textcolor{black}{of the $m^{\text{th}}$ iteration} each node takes an action according to current policy obtained through Q-learning, and the AP estimates $\boldsymbol{\pi_{m_n}}$. 
		\item If  $\Vert{ \boldsymbol \pi_{m}}-{\boldsymbol \pi_{m_{n}}}\Vert_2\geq \epsilon_1$ or $n>T$, broadcast $\boldsymbol{\pi_{m+1}}= \boldsymbol\pi_{m_{n}} $; \textcolor{black}{else go to step 2 and set $n\gets n+1$.}
		\item Update  $\boldsymbol{\bar \pi}_{m+1}$ using \eqref{eq:FPP3} and $n\gets 0$. 
		\een
		\doWhile{$\Vert{\boldsymbol{\bar \pi_{m+1}}}-{\boldsymbol{\bar \pi_{m}}}\Vert_2\leq \tilde{\epsilon}$.}\\
		\textbf{Output:} The stationary Nash maximizer policies and distribution are given by $\mathcal{F}^*$ and $\tilde{\pi}$, respectively.
	\end{algorithmic}\label{algorithm:MF_MARL}
\end{algorithm}

\textcolor{black}{Note that, in the Alogrithm~\ref{algorithm:MF_MARL}, ${\boldsymbol \pi}_{m_n}$ denotes the distribution of the nodes across the states, in the $n^{\text{th}}$ slot of the $m^{\text{th}}$ iteration. Further, ${\mathcal{F}_m^k}^*$ denotes the Nash maximizer policy of the $k^{\text{th}}$ node, at the $m^{\text{th}}$ iterations. The maximum duration of each iteration of Algorithm~\ref{algorithm:MF_MARL} is set to $T$. However, in the $n^{\text{th}}$ slot of the $m^{\text{th}}$ iteration, where $n<T$, the AP can terminate the current iteration by broadcasting the belief about the mean-field distribution $\boldsymbol{\pi}_{m_n}$, depending on the update rule in Step 3 of Algorithm~\ref{algorithm:MF_MARL}, i.e., when the  previous belief of the nodes about the distribution, $\boldsymbol{\bar \pi_{m}}$, is outdated. Note that, at the start of each new iteration the Q-values are initialized with the Q-values at the end of previous iteration. } 

In order to implement the Q-learning algorithm, a node requires to know the reward, i.e., the sum-throughput, obtained in each slot. Since the reward function is same across the nodes, this could be accomplished by using an estimate of the distribution \textcolor{black}{in \eqref{eq:reward}}. In particular, each node uses its own policy and an estimate of the distribution to build an estimate of the reward obtained in each slot. Alternatively, in each slot the AP can directly broadcast the total number of bits successfully decoded by it. The latter method obviates the need to estimate the distribution of the nodes, albeit at a cost of higher feedback overhead. The latter method is essentially cooperative multi-agent Q-learning\cite{Sunehag_ICAAMS_2018} where nodes attempt to maximize a common reward function. In our simulations it is observed that the proposed MF-MARL based approach performs marginally better than the cooperative multi-agent Q-learning method. In steps~$2$ and $3$ of the Algorithm~\ref{algorithm:MF_MARL}, the AP builds an estimate\footnote{Since transmit power used by a node determine the state of the node, in each slot, the AP can estimate the state of each node based on the transmit power.} of ${\boldsymbol \pi}_{m_n}$ and periodically broadcasts it to the entire network. In the simulations, presented in Sec.~\ref{sec:sim}, we use the empirical distribution as an estimate of ${\boldsymbol \pi}_{m_n}$.  

\subsection{Implementation via Deep Reinforcement Learning} \textcolor{black}{At each node, we implement the reinforcement learning algorithm using the deep Q-learning\cite{Mnih2015_DQN} method where the Q-function is approximated using a deep neural network (DNN). In order to learn the Q-function, the DNN is successively trained using the problem data, and a fixed target network which provides the reference Q-values. The target Q-network is periodically updated using the weights of the current Q-network. For further details on the deep Q-learning with fixed target Q-networks we refer the readers to \cite{Mnih2015_DQN}.} This approach of using a DNN to learn Q-function has the following advantages: $(i)$ it obviates the need to discretize the state space, as the Q-function approximation learned using the DNN is continuous over the state space, whereas in conventional approach it is learned for discrete state-action pairs, and (ii) it is inherently faster, compared to the conventional approach of implementing the Q-learning. This is because for a given state the Q-function corresponding to all the actions is learned simultaneously. We also note that in the first and second step of Algorithm~\ref{algorithm:MF_MARL}, the use of Q-learning could be replaced by any other variant of reinforcement learning schemes, e.g., actor-critic algorithm. 

In the following section, we compare both the centralized and distributed approach from the energy consumption perspective, discuss their feasibility for the low power sensor nodes, and adapt the centralized DNN based approach developed in this section to develop an energy efficient distributed implementation. 
\section{Energy Efficient Distributed Power Control}
\label{sec:energy_compare}
First, we compare the energy required for implementation of both the methods, \textcolor{black}{proposed in the previous sections}. In order to do this, it is important to understand the energy consumption of a DNN\cite{Chen_sysML_feb2018}. As observed in the previous sections, the design of a DNN involves several hyper-parameters, e.g., number of layers, number of nodes in each layer, length of weight vectors for each node, etc., which are conventionally chosen to improve the accuracy of the DNN. These parameters also affect the energy consumption of a DNN, which also depends on the algorithm being implemented by the DNN\cite{Yang_Asilomar_Oct2017}. Thus, the energy consumed by the DNN is a complicated function of these parameters and is not possible to compute it beforehand. Note that, the energy consumed by a DNN is determined by not only the number of multiplication-and-accumulation  (MAA) operations that need to be performed by a DNN, but also by the memory hierarchy and the data movement\cite{Yang_Asilomar_Oct2017}. Indeed, as shown in Table~\ref{table_DNN_energy}, the energy consumption of a DNN is overwhelmingly dominated by the energy consumed for data movement. For instance, the energy cost incurred by a single dynamic RAM access is 200 times more than a MAA operation. Even an access to local on-board memory costs more than a MAA operation. Thus, an algorithm which requires a high number of memory accesses will incur a larger energy cost, in comparison to an algorithm which does not require any memory access during runtime.  
\begin{table}[t!]
	\renewcommand{\arraystretch}{1.3}
	\caption{An example of normalized values for energy consumption of memory access and computation \cite{Yang_Asilomar_Oct2017}. Here, the arithmetic and logic unit (ALU) contains register file (RF). The size of RF is smaller than a the processing engine (PE), which, in turn is smaller than a global buffer. The dynamic RAM (DRAM) is the largest among the all and is external to a DNN \cite{Sze_Proc_of_IEEE_Dec2017}.   }   
	\label{table_DNN_energy}
	\centering
	\begin{tabular}{|c|c|}
		\hline
		\begin{tabular}[x]{@{}c@{}}Hierarchy of Memory Access\end{tabular} &  \begin{tabular}[x]{@{}c@{}}Normalized energy cost\end{tabular}\\
		\hline
		MAA  &  1\bf{x}\\ 
		\hline
		RF $\to $ ALU  &  1\bf{x} \\
		\hline
		PE $\to $ ALU &   2\bf{x}  \\
		\hline
		Buffer $\to $ ALU  &  6\bf{x}   \\
		\hline
		DRAM $\to $ ALU  &  200\bf{x}  \\
		\hline
	\end{tabular}
\end{table}

We note that, unlike the deep Q-learning, in the centralized approach the DNN is trained only once, and the training data needs to be generated only once by solving multiple instantiations of the offline problem. Thus, for the centralized approach the DNN can be trained in the cloud, before deploying the trained DNN in an EHN. On the other hand, for deep Q-learning the DQN is trained successively during the operation of the algorithm. In addition, the deep Q-learning also requires to maintain a memory buffer, for experience replay, which is essential for the stability of the algorithm. This further adds to the energy cost of the deep Q-learning algorithm. 

For the centralized approach, once the DNN is trained and deployed, the proposed online power control policy only requires to perform MAA operations to generate the transmit power vector. In particular, it requires $\sum_{j=1}^{h+2}N_jN_{j-1}$ multiplications. Thus, it requires no external memory access for its operation which makes it more favorable for EHNs, compared to deep Q-learning. The number of MAA operations required to compute the output transmit power vector using the centralized approach can be further optimized by the use of model-compression methods \cite{Ye_Arxiv_2018, Yiwen_NIPS_2016} which attempt to further reduce the number of neurons and connections in the DNN, without compromising its accuracy. However, we emphasize that the design of DNN-based centralized policies is only possible in the scenarios where some a-priori knowledge about the EH process and the channel state is available. Thus, unlike the MF-MARL approach proposed in Sec.~\ref{sec:MF_MARL_Algo}, the DNN based policies can \emph{not} be used in the scenarios where no knowledge about the EH process and the channel state is available. However, due to the aforementioned energy concerns, it would be desirable to have a distributed implementation of the DNN based centralized approach, developed in the previous section,  for the applications where offline information about the EH process and the channel state is available.  The following  subsection describes how the proposed DNN based centralized online power control scheme could be modified to develop an energy efficient distributed power control policy. 

\subsection{A Low Energy Cost Decentralized Policy for EHNs }
\textcolor{black}{We note that, for the centralized DNN-based online policy, proposed in Sec. III, once the DNN is trained and deployed, the input vector to the trained DNN is constituted by the state of all the nodes in the network. A distributed implementation of this scheme could be facilitated by deploying the trained DNN, obtained after the centralized training, at the individual nodes. However, to locally determine the transmit powers at the nodes, each node would require to know the values of the energy harvested, battery and channel states of all the other nodes. Thusm the amount of overhead involved in the exchange of global state information across the network forbids the distributed implementation of the centralized approach. However, as observed for the MF-MARL algorithm, the optimal transmit power of a node only depends on the other nodes through distribution of nodes across the states, $\boldsymbol{\pi}$. The proposed distributed DNN based approach circumvents this problem by sampling the states of other nodes from the distribution of the nodes across the states, denoted by $\boldsymbol{\pi}$. Intuitively, given the trained DNN deployed at each EHN, the optimal performance can be obtained by constructing the input vector to the DNN by sampling the states of other nodes from the distribution $\boldsymbol{\pi}$. In particular, in the $n^{\text{th}}$ slot, the input vector at the $k^{\text{th}}$ EHN can be generated as $\left(e_{s_n}^{1}, B_{s_n}^{1}, g_{n_s}^{1},\ldots, e_n^k,B_n^k,g_n^k,\ldots e_{s_n}^K,B_{s_n}^K,g_{s_n}^K\right)$, where $(e_{s_n}^{1}, B_{s_n}^{1}, g_{n_s}^{1})$ denotes the state of the first node, sampled from the distribution $\boldsymbol\pi$ and $(e_n^k,B_n^k,g_n^k)$ denotes the state of the $n^{\text{th}}$  node.}

\textcolor{black}{ Note that, for this distributed DNN approach, unlike MF-MARL where the policy is updated at each iteration, the policy is fixed
	for the entire duration of the operation. However, similar to MF-MARL, the distribution $\boldsymbol{\pi}$ is estimated and updated according to \eqref{eq:FPP3}, and is periodically broadcasted by the AP. Thus, for this policy, given a fixed trained DNN at each node only the distribution $\boldsymbol{\pi}$ evolves over time which is also guaranteed to converge while operating under a fixed policy. This follows from the finite state space of the game which, under a fixed policy, results in a positive recurrent Markov chain, provided the EH process and the wireless channel are stationary and ergodic.}

Also, we explicitly observe that, although distributed, this approach still requires the generation of a training set for the off-line training of the DNN, which in turn requires to know several realizations of the channel and EH processes for all nodes beforehand. Instead, this is not required by the proposed MF-MARL method. In the following, we present the numerical results. 
\section{Numerical Results}
\label{sec:sim}
We consider an EH MAC with $K=5$ EH transmitters where each EHN harvests energy according to a non-negative truncated Gaussian distribution with mean $m$ and variance $v=3.5$, independently of the other nodes. The capacity of the battery at each transmitter is $B_{\max}=20$ and the maximum amount of energy allowed to be used for transmission in a slot is $P_{\max}=15$. Note that, the unit of energy is $10^{-2}$ J. In the following, we first describe the architecture of the DNN and the training setup used for learning the centralized and DQN policy. 
\subsection{DNN Architecture and Training}
To learn the centralized policy, we use a DNN with an input and output layer containing $3K$ and $K$ neurons, respectively. It consists of $30$ hidden layers, with first hidden layer containing $30K$ neurons. Each subsequent odd indexed hidden layer contains the same number of neurons as the previous even indexed layer, i.e., $N_j=N_{j-1}$ for $j\in\{3,\ldots,31\}$. For each even indexed hidden layer the number of neurons is decreased by $2K$, i.e., $N_{j}=N_{j-1}-2K$ for $j\in \{4,\ldots,30\}$. We note that, the input layer has the index $1$, and the indices of the first hidden layer and the output layer are $2$ and $32$, respectively. The activation function used is Leaky rectified linear unit (ReLu). To train the network we use the mean-square error as the loss function. Training data is generated by solving $10^4$ instantiations of the offline problem with the horizon length $N=20$. Thus, the training dataset contains $2 \times 10^5$ datapoints, out of which $40000$ data points are used for validation. The performance is evaluated by computing the rate per slot (RPS) over $10^6$ slots. For these $10^6$ slots, instantiations of the EH process and the channel are generated independently of the training data.    

At each node, both the deep Q network as well as the the fixed target network consist of $10$ hidden layers and one input and output layer. The input layer contains $3$ neurons, while the number of neurons in the output layer is equal to $|\mathcal{A}|=151$, where $\mathcal{A}=\{0,0.1,0.2,\ldots,15\}$. The first, third, fifth, seventh, and ninth hidden layer consists of $60, 58, 56, 54,$ and $52$ neurons, respectively. As for the DNN architecture used in the centralized approach, the number of neurons in each even indexed hidden layer remains same as in the previous odd indexed hidden layer. At each layer, except the output layer, the rectified linear unit (ReLu) is used as an activation function. The output layer uses a linear activation function, motivated by the fact that using an activation function that applies cut-off values could result in low training errors simply because the output power would be artificially constrained to lie in the interval $[0,P_{max}]$ and not as a result of a proper configuration of the hidden layers. Instead, a linear output activation function allows the DNN to learn whether the adopted  configuration of the hidden layers is truly leading to a small error or whether it needs to be still adjusted  through further training. 

The deep Q-learning algorithm uses $\gamma=0.99$, and uses the exploration probability  $\epsilon_{\max}=1$ at the start which decays to $\epsilon_{\min}=0.01$ with a decay factor equal to 0.995. The replay memory of length $2000$ is used. For all the experiments, DQN is trained with a batch size equal to 32. In Algorithm~\ref{algorithm:MF_MARL}, we use $\epsilon_1 = 0.01 $, $\tilde\epsilon =0.001$, and update frequency $T=1000$. 

\subsection{Performance of Centralized Policy}
We first benchmark the performance of the proposed DNN based centralized online policy, against the performance of the optimal offline policy proposed in \cite{Wang_JSAC_Mar2015}. In the centralized scheme, the online policy is learned by training a deep neural network using the data obtained by jointly optimal offline policies \cite{Wang_JSAC_Mar2015}.

\begin{table}[t!]
	\renewcommand{\arraystretch}{1.3}
	\caption{Performance of the DNN based policy for an EH MAC with $K=5$ users and $v=3.5$. Performance of the offline policy corresponds to 100\%. }   
	\label{table_DNN_perf}
	\centering
	\begin{tabular}{|c|c|c|c|c|}
		\hline
		\begin{tabular}[x]{@{}c@{}}Mean\\ (m)\end{tabular} &  \begin{tabular}[x]{@{}c@{}}Offline Policy\\ (RPS in nats)\end{tabular} & \begin{tabular}[x]{@{}c@{}}DNN policy\\ (RPS in nats)\end{tabular} & \begin{tabular}[x]{@{}c@{}}DNN policy\\ (Percentage )\end{tabular}\\
		\hline
		$4$  &  3.4907 &  3.1498 &  90.23\%\\ 
		\hline
		$5$  &  3.6564 &  3.3107 & 90.54\%\\
		\hline
		$6$  &   3.7877 & 3.4410 & 90.84\% \\
		\hline
		$7$  &  3.8922 & 3.5102 & 90.18\%  \\
		\hline
		$8$  &  3.9740 & 3.6146 & 90.95\% \\
		\hline
		$9$  &  4.0407 & 3.5676& 88.29\% \\
		\hline
	\end{tabular}
\end{table}

Table~\ref{table_DNN_perf} shows the performance of the proposed DNN based policy. The last column of the table presents the RPS as the percentage of the throughput achieved by the offline policy. It can be observed that the proposed policy achieves roughly $90\%$ of the throughput obtained by the offline policy. We note that, since an offline policy is designed using non-causal information, the proposed policy can \emph{not}  achieve the throughput obtained by the optimal offline policy. Note that, the MDP formulation of this problem is computationally intractable, due to \emph{state space of the size} of order $10^{12}$, even with the channel gains quantized to just $8$ levels. 

Table~\ref{table_P2P_mean10} compares the performance of the proposed DNN based policy against deep Q-learning and the MDP, for point-to-point links, i.e., $K=1$, with mean $m=10$. The proposed DNN based policy achieves approximately $98$~\% of the time-averaged throughput achieved by the offline policy. It is interesting to note that the throughput achieved by the proposed DNN based policies is \emph{marginally better} than the throughput achieved by the online policies designed using the deep Q-learning. Also, the proposed DNN based policies outperforms the MDP based policies which achieves only approximately $84$~\% of the throughput achieved by the offline policy. Theoretically, an online policy designed using the  MDP achieves the optimal performance. However, the performance of MDP policy degrades due to quantization of the state and action spaces. We note that the computational complexity for solving an MDP increases in direct proportion to the number of quantization levels used for state and action spaces. On the other hand, the proposed DNN based policy operates with continuous state and action spaces. In contrast, as shown in Table~\ref{table_policy_quantization}, while the DQN uses continuous state at the input, the output of DQN network is quantized which, in turn, results in performance loss, compared to the DNN-based policy.
Next, we compare the performance of the proposed MF-MARL 
and cooperative Q-learning approaches against the DNN based centralized and distributed policy. 
\begin{table*}[t!]
	\renewcommand{\arraystretch}{1.3}
	\caption{Performance of the DNN based online policy for a point-to-point link with $m=10$. The action space of DQN based policy is $\mathcal{A}\triangleq\{0,0.1,\ldots, 15\}$. On the other hand, the MDP based solution is obtained using the action space $\mathcal{A}\triangleq\{0,1,\ldots, 15\}$.}   
	\label{table_P2P_mean10}
	\centering
	\begin{tabular}{|c|c|c|c|c|}
		\hline
		\begin{tabular}[x]{@{}c@{}}Variance\\ (v)\end{tabular} &  \begin{tabular}[x]{@{}c@{}}Offline Policy\\  (RPS in nats)\end{tabular} & \begin{tabular}[x]{@{}c@{}}DNN Policy\\ (Percentage )\end{tabular}& \begin{tabular}[x]{@{}c@{}}DQN Policy\\ (Percentage )\end{tabular} & \begin{tabular}[x]{@{}c@{}}MDP Policy\\ (Percentage )\end{tabular}\\
		\hline
		$1$  &  2.0434 & 98.41\%  & 95.56\% & 83.32\%  \\ 
		\hline
		$2$  &  2.0375 & 98.56\%  & 95.24\% & 83.60\%   \\
		\hline
		$3$  &   2.0372 & 98.38\% & 98.11\% & 83.32\%  \\
		\hline
		$4$  &  2.0347 & 95.85\% & 96.54\% & 83.37\%  \\
		\hline
		$5$  &  2.0310 & 97.72\% &95.28\% & 83.29\%  \\
		\hline
		$6$  &  2.0284 & 98.22\%  &98.18\% & 83.21\%  \\
		\hline
	\end{tabular}
\end{table*}

\begin{table*}[t!]
	\renewcommand{\arraystretch}{1.3}
	\caption{Summary of the inputs and outputs of the policies compared in this section. Note that, since both the MF-MARL and the cooperative Q-learning uses deep Q-learning at each individual node, their input and output values are continuous and discrete, respectively  }   
	\label{table_policy_quantization}
	\centering
	\begin{tabular}{| *{3}{c|} }
		\hline
		Policy & Input & Output\\
		\hline
		Centralized DNN & Continuous & Continuous  \\
		\hline
		Distributed DNN  &  Discrete &  Continuous \\ 
		\hline
		Deep Q-learning  &  Continuous & Discrete\\
		\hline
		MDP  &  Discrete & Discrete\\
		\hline
	\end{tabular}
\end{table*}
\subsection{Performance of Distributed Policies}

\begin{table*}[t!]
	\renewcommand{\arraystretch}{1.3}
	\caption{Performance of the MF-MARL and cooperative multi-agent Q-learning approach for an EH MAC with $K=5$ users and $v=3.5$. Performance of the centralized policy corresponds to 100\%. }   
	\label{table_MARL}
	\centering
	\begin{tabular}{| *{8}{c|} }
		\hline
		\begin{tabular}[x]{@{}c@{}}Mean\\ (m) \end{tabular} &  \begin{tabular}[x]{@{}c@{}}Centralized Policy\\ (RPS in nats)\end{tabular} & \multicolumn{2}{c|}{\begin{tabular}[x]{@{}c@{}}MF-MARL policy\end{tabular}} & \multicolumn{2}{c|} {\begin{tabular}[x]{@{}c@{}}Cooperative Q-learning \end{tabular}} & \multicolumn{2}{c|} {\begin{tabular}[x]{@{}c@{}}Distributed DNN \end{tabular}}\\
		\cline{3-8}
		&  & RPS & \% & RPS & \% & RPS & \%  \\
		\hline
		$4$  &  3.1498 &  2.9390 & 93.30\% &  2.9354 & 93.19\% & 2.6788 & 85.04\%\\ 
		\hline
		$5$  &  3.3107 &  3.1311 & 94.57\% &  3.0046 & 90.75\% & 2.8918 & 87.34\%\\
		\hline
		$6$  &   3.4410 & 3.1072 & 90.29\% &  3.1852 & 92.56\%  & 3.0765 & 89.40\%\\
		\hline
		$7$  &  3.5102 & 3.2960 & 93.89\% &  3.2417 & 92.35\%  & 3.1388 & 89.41\%\\
		\hline
		$8$  &  3.6146 & 3.3973 & 93.98\% &  3.3064  & 91.47\% & 3.2518 & 89.96\%\\
		\hline
		$9$  &  3.6166 & 3.5179 & 95.68\% &  3.4528  & 93.90\% & 3.1922 & 88.26\%\\
		\hline
	\end{tabular}
\end{table*}

As observed from the results in Table~\ref{table_MARL}, the policies obtained using the proposed MF-MARL based approach achieve a sum-throughput which is close to the throughput achieved by the centralized policies. However, in contrast to the proposed approaches, the centralized online policy requires information about the state of all the nodes in the network. Note that, as shown in Table~\ref{table_policy_quantization}, in order to implement MF-MARL (or deep Q-learning), the actions space, $\mathcal{A}$, has to be quantized which leads to a loss in the throughput, compared to the centralized scheme where the output transmit powers are continuous. We observe that the proposed MF-MARL based approach performs marginally better than the cooperative multi-agent Q-learning based scheme. However, in contrast to the cooperative multi-agent Q-learning approach, the MF-MARL based procedure requires significantly less feedback. Also, it is interesting to note that the proposed MF-MARL algorithm achieves the near-optimal throughput even for a network with small number of nodes. 

Furthermore, we note that the distributed DNN approach proposed in Sec.~\ref{sec:energy_compare} also achieves throughput competitive to the MF-MARL approach. Both the MF-MARL and the distributed DNN approaches use the distribution vector $\boldsymbol{\pi}$ for their operation. The distribution vector $\boldsymbol{\pi}$ is estimated using the empirical distribution over the discretized (or quantized) state space. Since, in the distributed DNN approach, each EHN constructs the input vector to DNN by sampling the states of the other nodes from the distribution $\boldsymbol{\pi}$, hence the input states for other nodes are essentially sampled from the quantized state space (see Table~\ref{table_policy_quantization}). This is in contrast to the MF-MARL approach where the input states are continuous variables, however the output transmit power variables are quantized. Thus, similar to the MF-MARL algorithm, the distributed DNN approach also has a performance gap from the centralized policy, due to quantization, which reduces with a finer quantization. In the following, we study the impact of hyperparameters such as, update frequency, replay buffer size on the performance of the MF-MARL approach. We also present the results to show the speed of convergence of the MARL approaches.
\begin{figure}[t!]
	\centering
	\includegraphics[width=3.7in]{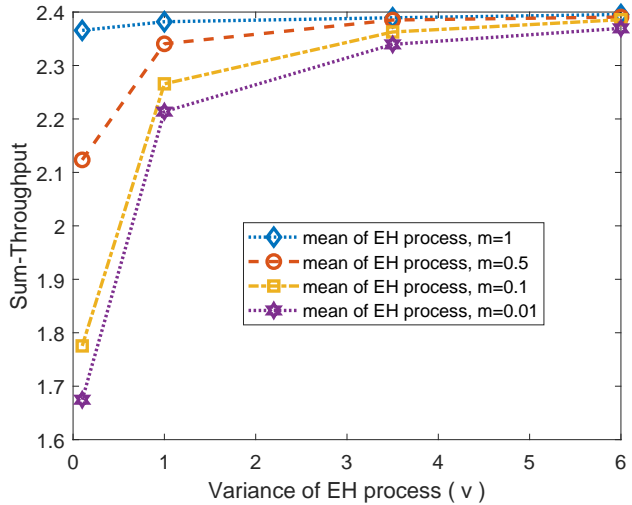}
	\caption{Sum-throughput achieved by MF-MARL for a fading EH MAC with $20$ users, i.e, $K=20$. The parameters used for this simulation are $B_{\max}=2$ and $P_{\max}=1$. Since the transmit power levels are quantized with a range of $0.1$, the transmit powers at each node are in the range $\{0,0.1 \cdots,1\}$.}
	\label{Fig:large_network}
\end{figure}

Further, the result in Fig.~\ref{Fig:large_network} illustrates the performance of the proposed MF-MARL approach for a fading EH MAC with $K=20$ users. In this scenario, the performance of the network is constrained by the limited capacity of the battery attached to the node. It is observed in our simulations that, even with $K=20$ nodes, our MF-MARL approach is able to learn the policies in a completely distributed fashion, and converges to a stable throughput. This could be observed by the fact that for $m=0.01$ the sum-throughput increases with the variance, i.e., the energy availability.  
\subsection{Convergence and Effect of Hyperparameters}     
Further, the results in Fig.~\ref{Fig:MF_MARL_covergence} show the throughput achieved by our MF-MARL algorithm as a function of slot index. It is interesting to observe that the MF-MARL algorithm converges very fast, i.e., within first $1000$ slots, for $m=7$ and $m=8$ the obtained throughput reaches within the $99\%$ of the throughput attained finally. Although, for $m=5$ and $m=9$ the MF-MARL learns at a relatively slower pace, yet within first $5000$ slots it achieves the throughput close to $95\%$ of the final value. A similar trend is observed for cooperative Q-learning.
\begin{figure}[t!]
	\centering
	\includegraphics[width=3.7in]{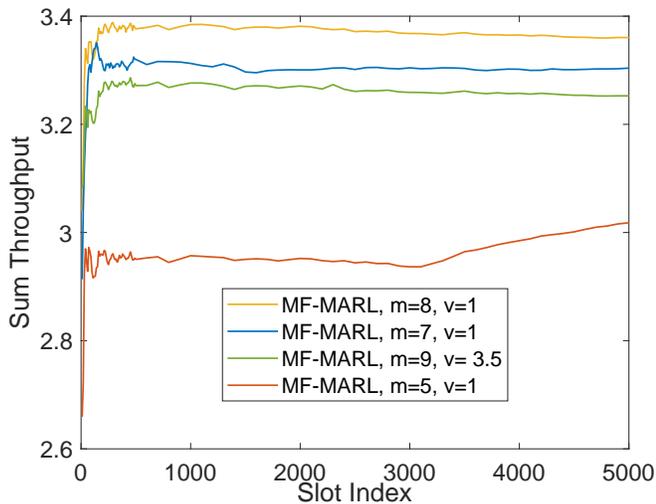}
	\caption{Convergence of MF-MARL algorithm. For $m=7$ and $m=8$, within first $1000$ slots the sum-throughput obtained by the MF-MARL approaches roughly $99\%$ of the value shown in Table~\ref{table_MARL}. Similarly, for $m=5$ and $m=9$ also, it takes approximately $5000$ slots for the MF-MARL to achieve a sum-throughput which is within $97\%$ and $94\%$ of the value attained finally. }, 
	\label{Fig:MF_MARL_covergence}
\end{figure}
\begin{figure}[t!]
	\centering
	\includegraphics[width=3.7in]{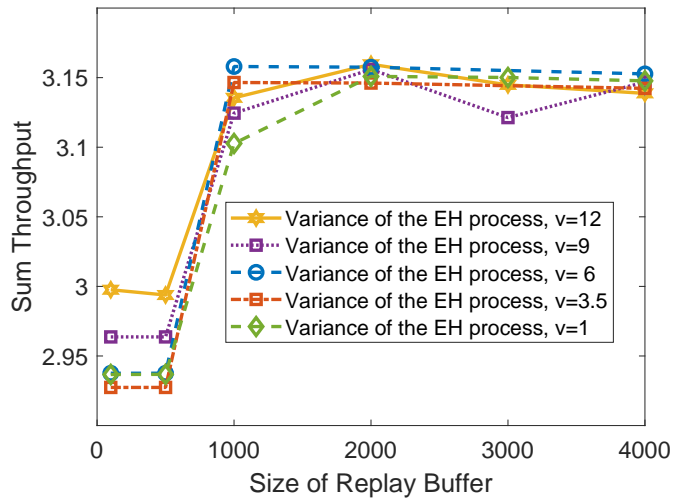}
	\caption{Impact of replay buffer size on the performance of MF-MARL. The mean of the harvesting process is $m=5$.}
	\label{Fig:MF_MARL_buffersize}
\end{figure}
\begin{figure}[t!]
	\centering
	\includegraphics[width=3.7in]{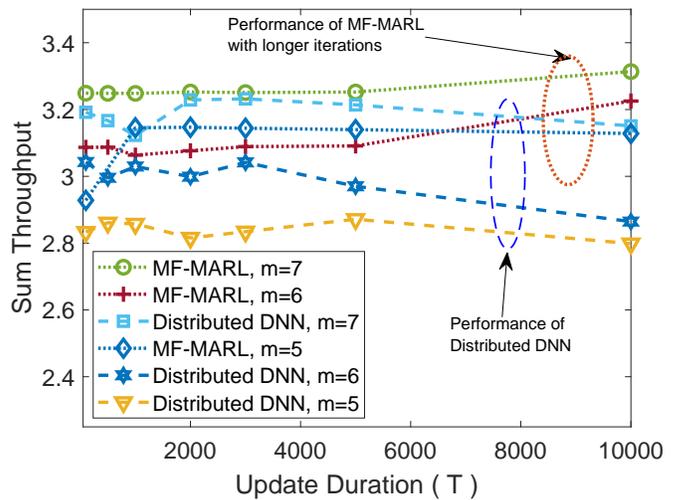}
	\caption{Impact of the update frequency on performance of the MF-MARL and the distributed DNN. The variance of the harvesting process is $v=3.5$. The results indicate that the performance of MF-MARL improves with longer update duration. On the other hand, performance of the distributed DNN based approach is relatively independent of the update duration.   }
	\label{Fig:MF_MARL_Update_frequency}
\end{figure}
The results shown in Fig.~\ref{Fig:MF_MARL_buffersize} illustrate the impact of the size of the replay buffer on the performance of the MF-MARL algorithm. From this plot it can be concluded that the size of replay buffer has a threshold effect on the throughput achieved by the MF-MARL algorithm. A small size replay buffer prohibits the algorithm from converging to the optimal throughput. However, beyond a sufficient size of the replay buffer the throughput does not improve further. The result in Fig.~\ref{Fig:MF_MARL_Update_frequency} shows the impact of parameter $T$, in the Algorithm~\ref{algorithm:MF_MARL}, on the sum-throughput achieved by the distributed policies. Recall that, $T$ determines the frequency with which the AP broadcasts the updates about the estimate of the distribution $\boldsymbol\pi$, hence is termed as update duration. For the MF-MARL, it is observed in the simulations that a longer update duration  result in an improved throughput. This is because the estimates obtained by computing the empirical distribution over larger number of slots are more accurate, which, in turn,  for the MF-MARL, leads to better estimates of the reward at each individual node. Consequently, it aids in the learning of Q-function, and leads to better DQN approximation. In contrast, performance of the distributed DNN approach is relatively independent of the update duration. This is because, the distributed DNN approach does not use the estimate of the distribution $\boldsymbol\pi$ for learning the policy, i.e., in the distributed DNN approach $\boldsymbol\pi$ is used only for sampling the states of the other nodes. Also, from \eqref{eq:FPP3}, regardless of the update frequency, over the time, iterative estimates of the distribution $\boldsymbol\pi$ converge, and the states of the other nodes are sampled from the correct distribution. In contrast, for the MF-MARL, estimates of $\boldsymbol \pi$ are critically used for learning, therefore the estimation inaccuracies jeopardize the learning procedure, and may adversely affect the sum-throughput.   

\begin{figure}[t!]
	\centering
	\includegraphics[width=3.7in]{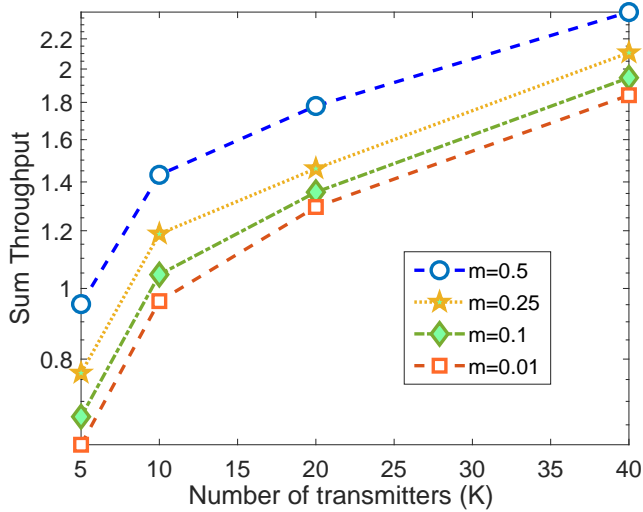}
	\caption{Sum throughput of the network as a function of number of nodes. Each transmitter contains a unit size battery, i.e., $B_{\max}^k=1$ for all $1\leq k\leq K$. The transmit power levels are restricted to binary levels, $\mathcal{A}=\{0,1\}$, and the variance of the harvesting process is $v=0.1$. }   
	\label{Fig:sum_throughput_number_nodes}
\end{figure}

The results in Fig.~\ref{Fig:sum_throughput_number_nodes} illustrate the variations in the sum-throughput achieved by the MF-MARL, as a function of the number of transmitters in the network. As expected, the sum-throughput of the network increases with both the number of transmitters in the network, $K$, as well as with the mean of the harvesting process, $m$. This shows the that the proposed MF-MARL apporach can learn  effectively, even in a network where the number of transmitters is large.      
\section{Conclusions}
\label{sec:Concl}
In this paper, we proposed a mean-field multi-agent reinforcement learning based framework to learn the optimal power control to maximize the throughput of large EH MAC.
First, we modeled the throughput maximization problem as a discrete-time MFG and analytically established that the game has a unique stationary solution. We proposed a reinforcement learning based procedure to learn the stationary solution of the game and established the convergence of the proposed procedure. Next, to benchmark the performance of the distributed power control policies, obtained using the proposed MF-MARL framework, we also developed a DNN based centralized online power control scheme. The centralized power control approach learns the optimal online decision rule using the data obtained through the solution of offline policies. The numerical results demonstrated that both the centralized as well as the distributed power control schemes achieve a throughput close to the optimal.  
\appendix
\section{Proof of Theorem~\ref{thm:unique_stationary_soln_GT}}
\label{app:proof_thm_unique_stationary_soln_GT}
\begin{proof}
	The proof follows directly from the result in Theorem~\ref{thm:unique_stationary_solution}, provided there exists a unique Nash maximizer and the reward function is monotone in variable $\boldsymbol{\pi}$. The uniqueness of Nash maximizer can established using the result in Theorem~\ref{thm:unique_Nash_maximizer}. It is easy to verify that the reward and value function of the game $\mathcal{G}_T$ satisfies the strictly diagonally concavity property. In order to complete the proof we just need to show that the reward function is monotone with parameter $\boldsymbol \pi$, i.e.,  
	\begin{align}
	\sum_{i=1}^d(\pi_i^2- \pi_i^1)(\mathcal{R}_i(\mathcal{P},\pi^2)-\mathcal{R}_i(\mathcal{P},\pi^1))\geq 0. 
	\end{align}	
The proof follows by noting the fact that since the reward obtained by a node does not depend on the state of the node, i.e., $\mathcal{R}_i(\mathcal{P},\pi^2)=\mathcal{R}(\mathcal{P},\pi^2)$. Hence, the RHS in the above can be expressed as 	$(\mathcal{R}(\mathcal{P},\pi^2)-\mathcal{R}(\mathcal{P},\pi^1))\left(\sum_{i=1}^d\pi_i^1-\sum_{i=1}^d\pi_i^1\right)=0$. 
\end{proof}

\bibliographystyle{IEEEtran}
{\bibliography{bibs/IEEEabrv,bibs/bibJournalList,bibs/references}}
\end{document}